\documentclass[11pt]{article}

\title{Multiplicative Oracle Inequalities for Transductive Learning via Level-Set Aggregation}
\usepackage{times}

\author{%
Jian Qian\\
The University of Hong Kong\\
\texttt{jianqian@hku.hk}
\and
Jiachen Xu\\
Hong Kong University of Science and Technology\\
\texttt{jxuec@connect.ust.hk}
}
\date{}

\newcommand{\BlackBox}{\rule{1.5ex}{1.5ex}}
\usepackage{amsthm}
\renewenvironment{proof}[1][Proof]{\par\noindent{\bf #1\ }}{\hfill\BlackBox\\[2mm]}

\newcommand{\nonl}{\renewcommand{\nl}{\let\nl}}

\usepackage[utf8]{inputenc}
\usepackage[T1]{fontenc}
\usepackage[margin=1in]{geometry}
\usepackage{amsfonts}
\usepackage{amssymb}
\usepackage[round]{natbib}

\usepackage{mathtools}
\usepackage{mathrsfs}
\usepackage{xcolor}
\usepackage{hyperref}
\usepackage{booktabs}
\usepackage{bbold}
\usepackage{bm}
\usepackage{enumitem}
\usepackage{framed}
\usepackage{comment}
\usepackage{bbm}
\usepackage{etoc}
\usepackage{ifthen}

\usepackage{algorithm}
\usepackage{algpseudocode}

\usepackage{cleveref}
\usepackage{aliascnt}

\newenvironment{keywords}
{\par\smallskip\noindent\textbf{Keywords:} }
{\par\smallskip}

\newtheorem{theorem}{Theorem}

\makeatletter
\@ifundefined{lemma}{}{%

}
\makeatother

\newaliascnt{lemma}{theorem}
\newtheorem{lemma}[lemma]{Lemma}
\aliascntresetthe{lemma}

\crefname{lemma}{Lemma}{Lemmas}
\Crefname{lemma}{Lemma}{Lemmas}

\makeatletter
\@ifundefined{proposition}{}{%

}
\makeatother

\newaliascnt{proposition}{theorem}
\newtheorem{proposition}[proposition]{Proposition}
\aliascntresetthe{proposition}

\crefname{proposition}{Proposition}{Propositions}
\Crefname{proposition}{Proposition}{Propositions}

\makeatletter
\@ifundefined{definition}{}{%

}
\makeatother

\newaliascnt{definition}{theorem}
\newtheorem{definition}[definition]{Definition}
\aliascntresetthe{definition}

\crefname{definition}{Definition}{Definitions}
\Crefname{definition}{Definition}{Definitions}

\makeatletter
\@ifundefined{remark}{}{%

}
\makeatother

\newaliascnt{remark}{theorem}
\newtheorem{remark}[remark]{Remark}
\aliascntresetthe{remark}

\crefname{remark}{Remark}{Remarks}
\Crefname{remark}{Remark}{Remarks}

\makeatletter
\@ifundefined{assumption}{}{%

}
\makeatother

\newaliascnt{assumption}{theorem}
\newtheorem{assumption}[assumption]{Assumption}
\aliascntresetthe{assumption}

\crefname{assumption}{Assumption}{Assumptions}
\Crefname{assumption}{Assumption}{Assumptions}

\makeatletter
\@ifundefined{corollary}{}{%

}
\makeatother

\newaliascnt{corollary}{theorem}
\newtheorem{corollary}[corollary]{Corollary}
\aliascntresetthe{corollary}

\crefname{corollary}{Corollary}{Corollaries}
\Crefname{corollary}{Corollary}{Corollaries}

\renewcommand{\cref}{\Cref}
\newcommand{\pfref}[1]{Proof of \Cref{#1}}

\addtocontents{toc}{\protect\setcounter{tocdepth}{0}}

\definecolor{jmlrblue}{HTML}{0000BB}

\hypersetup{
  colorlinks=true,
  linkcolor={jmlrblue},
  citecolor={jmlrblue},
  urlcolor={jmlrblue}
}

\newcommand{\bR}{\mathbb{R}}

\newcommand{\norm}[1]{\lVert {#1} \rVert}

\renewcommand{\leq}{\leqslant}
\renewcommand{\geq}{\geqslant}
\renewcommand{\le}{\leqslant}
\renewcommand{\ge}{\geqslant}
\newcommand{\argmin}{\mathop{\mathrm{arg}\,\mathrm{min}}}

\newcommand{\wh}{\widehat}

\renewcommand{\ln}{\log}

\newcommand{\eps}{\varepsilon}

\newcommand{\ldef}{\vcentcolon=}

\newcommand{\cG}{\mathcal{G}}
\newcommand{\cX}{\mathcal{X}}
\newcommand{\cY}{\mathcal{Y}}

\newcommand{\cH}{\mathcal{H}}

\usepackage{xspace}

\newtheorem*{theorem*}{Theorem}

\DeclarePairedDelimiter{\prn}{(}{)}

\DeclarePairedDelimiter{\set}{\{}{\}}

\definecolor{myRed}{HTML}{D24D5C}
\definecolor{myGreen}{RGB}{46,139,87}

\usepackage{color-edits}
\addauthor{jq}{myRed}
\addauthor{jx}{myGreen}

\newcommand{\cT}{\mathcal{T}}

\newcommand{\agg}{\mathrm{Agg}}
\newcommand{\R}{\mathbb{R}}
\newcommand{\tol}{t}
\newcommand{\med}{\mathrm{med}}
\newcommand{\cP}{\mathcal{P}}
\newcommand{\yhat}{\widehat{y}}

\newcommand{\mlsa}{\textsc{MLSA}\xspace}
\newcommand{\maxtol}{\tol_{\max}}

\begin{document}

\maketitle

\begin{abstract}
We revisit transductive learning where predictions are made with the set of all covariates known in advance.
In the leave-one-out (LOO) setting, the prediction is made with labels of the remaining sample points and evaluated by the average error.
In particular, we study multiplicative oracle
inequalities for agnostic transductive LOO prediction for a variety of tasks, including classification with 0-1 loss, squared loss regression, density estimation, and logistic regression. \looseness=-1

Specifically, we introduce \emph{Median of Level-Set Aggregation} (MLSA), an aggregation procedure built on near--ERM level sets (i.e., empirical-risk level sets around the ERM).
We prove a general multiplicative oracle inequality for the LOO error of the form
\[
\mathrm{LOO}_S(\mlsa)
\;\le\;
C \prn*{ \frac{1}{n}\min_{h\in\cH} L_S(h)
\;+\;
\frac{\log|\cH|}{n}},
\qquad C>1,
\]
where $\cH$ is the hypothesis/function class.
This inequality holds for hypothesis classes under a local level-set growth condition together with losses satisfying a mild monotonicity assumption.
For classification with VC classes under the $0$--$1$ loss, the $\log |\cH|$ factor can be improved to be $d\log n$, where $d$ is the VC dimension, recovering \cite{long1998complexity} up to a $\log n$ factor.
For logistic regression with bounded covariates and parameters, the $\log |\cH|$ factor can be improved to be $d\log n$ up to problem-dependent factors, where $d$ is the ambient dimension.\looseness=-1
\end{abstract}


\section{Introduction}
\label{sec:intro}

Transductive leave-one-out (LOO) prediction is a finite-sample problem.  We are given a labeled sequence
\(S=\{(x_i,y_i)\}_{i=1}^n\).  For each index \(i\), the learner observes all covariates \(x_1,\ldots,x_n\) and all labels except \(y_i\), and must predict the missing label at the observed covariate \(x_i\).  The loss is averaged over the \(n\) possible held-out indices:
\[
    \mathrm{LOO}_S(\mathcal A)
    :=
    \frac1n\sum_{i=1}^n \ell(\hat y_i,y_i),
\]
where \(\hat y_i\) is produced without seeing \(y_i\) and $\mathcal A$ is the algorithm.  There is no distribution in this definition.  The guarantee is about the fixed sample \(S\) itself.

The transductive viewpoint goes back to early work of
\citet{vapnik1974theory,vapnik1982estimation,vapnik1995nature}. For binary classification in the LOO setup,
\citet{haussler1994predicting} introduced the one-inclusion-graph
(OIG) method and obtained the benchmark \(d/n\) rate for realizable
VC classes.  \citet{long1998complexity} extended the
graph-orientation idea to arbitrary binary labelings and obtained a
multiplicative oracle inequality for agnostic \(0\)--\(1\)
classification, with a constant-factor comparator term and a
remainder of order \(d/n\).  A broad literature has since developed
OIG and related transductive methods for various domains including multiclass, partial,
scale-sensitive, and agnostic learning
\citep{rubinstein2009shifting,daniely2014optimal,
bartlett1998prediction,attias2023optimal,
asilis2024regularization,dughmi2025transductive}.

LOO ideas have also been developed for regression and statistical estimation.  
\citet{forster2002relative}
gave a general route from LOO guarantees to expected
instantaneous-loss guarantees and instantiated it for Gaussian and
Bernoulli density estimation and linear regression. For linear regression, the proposed estimator achieved the benchmark $d/n$ rate for $d$ dimensional space.
\citet{zhang2003leave} used LOO quantities to analyze regularized
kernel predictors through stability and RKHS structure.  More
recently,
\citet{mourtada2022improper,mourtada2021distributionfree} developed sharp procedures for misspecified
density estimation, logistic regression, and robust linear regression.  For logistic regression, \citet{mourtada2022improper} achieved $(ed+r^2R^2)/n$ rate where $r$ and $R$ bound the norm of the covariates and parameters.

Another thread of literature studies reductions between
transductive and PAC learning. \citet{warmuth2004optimal} asked whether optimal transductive learners can be used as optimal PAC learners.  This question led to line of work including \cite{wu2022expected,adenali2023optimal,adenali2023notalways,dughmi2025transductive}. In the agnostic
setting, this comparison is usually phrased in terms of additive
excess error and leads to the slower square-root scale; for binary
classification, for example, the benchmark rate is of order
\(\sqrt{d/n}\).  This is a different objective from the
complexity-over-\(n\) remainder pursued here.

Our goal is to obtain multiplicative, fixed-sample LOO guarantees through a
single method that applies beyond classification.  We introduce
\emph{Median of Level-Set Aggregation} (\mlsa), a unified construction for
hypothesis classes and losses satisfying a mild monotonicity condition.
Across all of our applications, the algorithm has the same two-stage form:
it aggregates predictions within near-ERM level sets and then takes a median
across tolerance levels.  Only the loss-specific aggregation rule and the
control of level-set growth change from one setting to another.

The guarantee we seek is a multiplicative oracle
inequality.  In its simplest form,
\begin{equation}
    \label{eq:intro-multiplicative-oracle}
    \mathrm{LOO}_S(\mathcal A)
    \leq
    C\prn*{
        \frac1n\min_{h\in\cH} L_S(h)
        +
        \frac{\mathrm{comp}(\cH)}{n}
    },
    \qquad C>1,
\end{equation}
where
\(L_S(h)=\sum_{i=1}^n\ell(h(x_i),y_i)\) and
\(\mathrm{comp}(\cH)\) is the relevant complexity term.  The first
term tracks the performance of the best comparator on the same
sample.  The second is a fast complexity-over-\(n\) remainder.
Thus the framework trades a constant factor on the comparator loss
for a fast remainder and a common construction across several
losses.

The algorithm is simple.  For each leave-one-out sample \(S_{-i}\)
and each tolerance level \(t\), we form the set of hypotheses whose
empirical loss on \(S_{-i}\) is within \(t\) of the optimum.  We
aggregate the predictions in this level set at \(x_i\), and then take
a median over the tolerance levels.  The analysis uses one
structural condition: empirical level sets should not grow too
quickly when the tolerance is enlarged.  Once this local growth
condition is verified, the same argument gives the LOO oracle
inequality.

Specialized OIG methods and linear square-loss procedures give the
sharpest benchmarks in binary classification and linear regression.
In those cases, our results show how the level-set principle connects
to known fast-rate behavior.  The broader contribution is that the
same MLSA framework also yields fixed-sample multiplicative LOO
guarantees for finite classes under bounded convex losses, finite
density classes under log loss, and bounded logistic regression.

Our contributions are as follows.
\begin{enumerate}[leftmargin=2em]
    \item We introduce \mlsa\ and prove a deterministic fixed-sample
    LOO oracle inequality under a local level-set growth condition together with a mild monotonicity assumption. This algorithmic principle is general and valid across all the tasks considered in previous literature. \looseness=-1
    \item We obtain fast rates comparable, up to a logarithmic factor, to the state-of-the-art results in all the tasks considered.
    \begin{itemize}[leftmargin=1em]
        \item For binary classification with VC dimension \(d\), we
        obtain
        \[
            \mathrm{LOO}_S(\mlsa)
            \leq
            C\prn*{\frac{1}{n}\min_{h\in\cH}L_S(h)
            +
            \frac{d\log n}{n} }.
        \]
        This is the basic discrete example of the framework and recovers
        the fast \(d/n\) scale up to a log term \citep{long1998complexity}. 
        \item For finite hypothesis classes and bounded convex losses,
        we obtain a multiplicative LOO oracle inequality with logarithmic
        dependence on \(|\cH|\). 
        \[
            \mathrm{LOO}_S(\mlsa)
            \leq
            C \prn*{\frac{1}{n}\min_{h\in\cH}L_S(h)
            +
            \frac{\log |\cH|}{n}}.
        \]
        This covers finite-class regression
        problems, including bounded square loss. The general result is new to the best of our knowledge.
        In the special case of bounded linear regression, this recovers the fast \(d/n\)-type scale up to a log term \citep{forster2002relative}.
        \item For finite density classes under log loss, we prove the
        corresponding oracle inequality under bounded log-likelihood
        ratios and give a smoothing construction that enforces this
        boundedness. This result is new to the best of our knowledge.
        \item For bounded logistic regression, we control the growth of empirical-risk level sets through their geometry.  
        This gives an
        additive term of order
        $O((r + \sqrt{rR/\lambda_{\min}(A)}) R\,d\log(nrR)/n)$, where \(r\) bounds the parameter norm,
        \(R\) bounds the covariate norm, and $A$ is the empirical covariance matrix. Our result is comparable to the best known result from \citet{mourtada2022improper}.
    \end{itemize}
\end{enumerate}

The rest of the paper is organized as follows.  \Cref{sec:setup} formalizes the transductive LOO setting and the multiplicative oracle-inequality objective.  \Cref{sec:overview} presents \mlsa and proves the general level-set theorem.  \Cref{sec:classification,sec:regression,sec:density-estimation,sec:logistic} instantiate the theorem for classification, bounded convex losses, density estimation, and logistic regression.

\section{Problem Setup}
\label{sec:setup}

\paragraph{Notation} For any integer $n$, $[n]\ldef \set{1,2,...,n}$. $\mathbf{1}(E)$ is the indicator function for $E$. For any finite set of values $\set{y_t}_{t\in \cT}$, $\med(\set{y_t}_{t\in \cT}) \ldef \inf \prn*{ \argmin_{y\in \bR}\sum_{t\in \cT}|y_t-y|}$ denotes the smallest median of the set.  \looseness=-1

Let $S=\{(x_i,y_i)\}_{i=1}^n$ be an arbitrary sequence with covariates
$x_i\in\cX$ and responses $y_i\in\cY$.
Fix a hypothesis class $\cH$ and a loss function
$\ell:\cY\times\cY\to\R_+$.
A learning algorithm $\mathcal{A}$ maps any training set $S'\subseteq S$ to a predictor
$h_{S'}\in\cH$. Note that in our transductive setup $\mathcal{A}$ always has access to the full covariates set $\set{x_i}_{i\in [n]}$.

For each $i\in[n]$, let $S_{-i}:=S\setminus\{(x_i,y_i)\}$ and denote by
$h_{S_{-i}}$ the predictor obtained by applying $\mathcal{A}$ to $S_{-i}$.
The leave-one-out error of $\mathcal{A}$ on $S$ is
\[
\mathrm{LOO}_S(\mathcal{A})
:= \frac{1}{n}\sum_{i=1}^n \ell\bigl(h_{S_{-i}}(x_i),y_i\bigr).
\]
For any predictions $\{\hat y_i\}_{i=1}^n$, we also write
\[
\mathrm{LOO}_S(\{\hat y_i\}_{i\in [n]})
:= \frac{1}{n}\sum_{i=1}^n \ell(\hat y_i,y_i).
\]

Our objective is to minimize the LOO error. This framework subsumes standard learning problems,
including binary classification (\cref{sec:classification}), regression (\cref{sec:regression}),
density estimation (\cref{sec:density-estimation}), and logistic regression
(\cref{sec:logistic}), which we instantiate in the corresponding sections.

For a hypothesis $h\in\cH$, define the empirical risks
\[
L_S(h) := \sum_{i=1}^n \ell(h(x_i),y_i),
\qquad
L_{S_{-i}}(h) := \sum_{j\neq i} \ell(h(x_j),y_j).
\]

We study guarantees of the form
\begin{equation}
\label{eq:multiplicative-oracle}
\mathrm{LOO}_S(\mathcal{A})
\;\le\;
C \prn*{ \frac{1}{n} \min_{h\in \cH}L_S(h)
\;+\;
\frac{\mathrm{comp}(\cH)}{n}},
\qquad C>1.
\end{equation}

\section{Median of Level-Set Aggregation}
\label{sec:overview}

In this section, we introduce the \emph{Median of Level-Set Aggregation} (\mlsa)
procedure, formalized in \cref{alg:levelset-aggregation}.
The algorithm consists of two aggregation layers: an inner layer produces leave-one-out predictions by aggregating hypotheses
lying in suitable empirical level sets, and an outer median aggregation over
a grid of tolerance levels.

For a tolerance level $t\ge 0$, define the full-sample and leave-one-out
level sets
\[
\cH_t
:= \{h\in\cH: L_S(h)\le \min_{h\in \cH}L_S(h)+t\},
\qquad
\cH_{t,i}
:= \{h\in\cH: L_{S_{-i}}(h)\le \min_{h\in \cH} L_{S_{-i}}(h)+t\}.
\]
For each index $i$ and tolerance $t$, the algorithm forms a prediction by
aggregating the evaluations of hypotheses in $\cH_{t,i}$ at $x_i$,
\[
\hat y_{t,i}
:= \agg(\cH_{t,i},x_i)
:= \agg\bigl(\{h(x_i): h\in\cH_{t,i}\}\bigr),
\]
where the aggregation rule $\agg$ is chosen according to the loss and output
structure.

In general, we require the aggregation rule to be stable, in the
sense that the loss of the aggregated prediction is controlled by the average
loss of the individual hypotheses being aggregated.

\begin{assumption}
\label{ass:agg}
Let $\mu$ be a measure on $\cH$.
For a fixed constant $c>0$, all $\cG\subseteq\cH$ and all $(x,y)$,
\[
\ell(\agg(\cG,x),y)
\;\le\;
\frac{c}{\mu(\cG)}\int_{\cG}\ell(h(x),y)\,\mu(dh).
\]
\end{assumption}

This assumption is mild and is satisfied by standard aggregation rules,
including majority vote for classification and averaging for convex losses.

The algorithm then aggregates across tolerance levels by taking the median of
the predictions $\{\hat y_{t,i}\}_{t\in\cT}$ over a prescribed tolerance grid
$\cT\subset\R_+$. The full procedure is summarized in
\cref{alg:levelset-aggregation}.

\begin{algorithm}[H]
\caption{Median of Level-Set Aggregation (\mlsa)}
\label{alg:levelset-aggregation}
\begin{algorithmic}[1]
\Require Sample $S$, hypothesis class $\cH$, loss $\ell$,
aggregation rule $\agg$, tolerance set $\cT\subset\R_+$
\For{$i=1,\dots,n$}
  \For{$t\in\cT$}
    \State Compute $\cH_{t,i}$
    \State $\hat y_{t,i} \gets \agg(\cH_{t,i},x_i)$
  \EndFor
  \State $\hat y_i \gets \med(\{\hat y_{t,i}\}_{t\in\cT})$
\EndFor
\State \Return $\{\hat y_i\}_{i=1}^n$
\end{algorithmic}
\end{algorithm}

In the next sections, we identify conditions under which these two aggregation
layers are effective. In \cref{sec:local-level-set-growth}, we show that the
inner aggregation layer satisfying \cref{ass:agg} yields meaningful guarantees
whenever the relevant level sets exhibit controlled growth in size
(\cref{ass:key}). In \cref{sec:median-aggregation-over-tol}, we show that the
outer aggregation layer resolves the choice of tolerance: if a strict majority
of tolerance levels in a prescribed grid satisfy the same growth condition
(\cref{ass:grid-key}), then median aggregation over the grid yields a valid
guarantee.

\subsection{Local Level-Set Growth}
\label{sec:local-level-set-growth}

\begin{assumption}[Local level-set growth]
\label{ass:key}
Let $\mu$ be a measure on $\cH$.
For tolerance $t$, gap $\Delta>0$, and constant $C_g\ge1$, we say that $(\cH,\ell)$ satisfies
the local level-set growth condition with parameters $(\mu,t,\Delta,C_g)$ if
\[
\frac{\mu(\cH_{t+\Delta})}{\mu(\cH_{t-\Delta})}\le C_g \quad\text{and}\quad \cH_{t-\Delta}\subseteq \cH_{t,i}\subseteq \cH_{t+\Delta}, \quad\text{for all~}i\in [n].
\]
\end{assumption}

Intuitively, this condition requires that the size of the level set does not increase too
rapidly when the tolerance is locally perturbed.
Throughout the paper, $\mu$ will be a counting measure for finite classes or an appropriate
volume measure in continuous settings. The gap $\Delta$ is chosen to upper bound the maximal
single-sample loss, which enforces the inclusion condition $\cH_{t-\Delta}\subseteq \cH_{t,i}\subseteq \cH_{t+\Delta}$ (\cref{lem:level-set-sandwich}).

\begin{lemma}
\label{lem:level-set-sandwich}
Assume $0\le \ell(y,y')\le \Delta$ for all $y,y'\in\cY$.
Then for every $i\in[n]$ and every $t\ge 0$, $\cH_{t-\Delta}
\;\subseteq\;
\cH_{t,i}
\;\subseteq\;
\cH_{t+\Delta}$.
\end{lemma}

\begin{proof}[\pfref{lem:level-set-sandwich}]
Define $h^\star\in\arg\min_{h\in\cH}L_S(h)$ and $h_{-i}^\star\in\arg\min_{h\in\cH}L_{S_{-i}}(h)$. By optimality of $h^\star_{-i}$ for $S_{-i}$,
\[
L_{S_{-i}}(h^\star_{-i})
\le L_{S_{-i}}(h^\star)
= L_S(h^\star) - \ell(h^\star(x_i),y_i)
\le L_S(h^\star).
\]
Conversely, by optimality of $h^\star$ for $S$,
\[
L_S(h^\star)
\le L_S(h^\star_{-i})
= L_{S_{-i}}(h^\star_{-i}) + \ell(h^\star_{-i}(x_i),y_i)
\le L_{S_{-i}}(h^\star_{-i}) + \Delta.
\]
Combining,
\begin{align}
\label{eq:level-set-sandwich}
   L_S(h^\star) - \Delta
\;\le\;
L_{S_{-i}}(h^\star_{-i})
\;\le\;
L_S(h^\star). 
\end{align}
Let $h\in\cH_{t,i}$. Then
\[
L_{S_{-i}}(h)
\le L_{S_{-i}}(h^\star_{-i}) + t.
\]
Adding $\ell(h(x_i),y_i)\le\Delta$ gives
\[
L_S(h) \le L_{S_{-i}}(h) +\Delta
\le L_{S_{-i}}(h^\star_{-i}) + t + \Delta.
\]
Using~\Cref{eq:level-set-sandwich},
\[
L_S(h)
\le L_S(h^\star) + t + \Delta.
\]
so $h\in\cH_{t+\Delta}$. On the other hand,
let $h\in\cH_{t-\Delta}$. Then
\[
L_S(h)
\le L_S(h^\star) + t - \Delta.
\]
Subtracting $\ell(h(x_i),y_i)\ge0$ yields
\[
L_{S_{-i}}(h)
\le L_S(h^\star) + t - \Delta.
\]
\Cref{eq:level-set-sandwich} gives,
\[
L_S(h^\star)
\le L_{S_{-i}}(h^\star_{-i}) + \Delta,
\]
hence
\[
L_{S_{-i}}(h)
\le L_{S_{-i}}(h^\star_{-i}) + t.
\]
Thus $h\in\cH_{t,i}$. This concludes our proof.
\end{proof}

\begin{proposition}
\label{prop:good-level-guarantee}
Suppose $\agg$ satisfies \cref{ass:agg} and that $\cH$ and $\ell$ satisfy \cref{ass:key} with parameter $(\mu,t,\Delta,C_g)$.
Then 
\[
\mathrm{LOO}_S(\{\hat y_{t,i}\}_{i\in[n]})
\;\le\;
\frac{cC_g}{n}\bigl( \min_{h\in \cH}L_S(h)+t+\Delta\bigr).
\]
\end{proposition}

In words, whenever the leave-one-out level set $\cH_{t,i}$ is sandwiched between two
full-sample level sets with controlled local growth, aggregating over $\cH_{t,i}$
incurs a leave-one-out error comparable, up to the growth factor $C_g$, to the empirical
risk at level $t$.

\begin{proof}[\pfref{prop:good-level-guarantee}]
By the assumption that $\cH_{\tol-\Delta} \subseteq \cH_{\tol,i} \subseteq \cH_{\tol+\Delta}$ and the assumption on the aggregation, we have
\begin{align*}
    \mu(\cH_{\tol-\Delta}) \sum_{i} \ell(\yhat_{\tol,i},y_i) &\leq  \sum_{i} \mu(\cH_{\tol,i})\ell(\yhat_{\tol,i},y_i) \\
    &\leq c\sum_i \int_{h\in \cH_{\tol,i}} \ell(h(x_i),y_i) \mu(d h)\leq c\sum_i \int_{h\in \cH_{\tol+\Delta}} \ell(h(x_i),y_i) \mu(d h).
\end{align*}
Then, by the definition of the level sets, we further have
\begin{align*}
    \int_{h\in \cH_{\tol+\Delta}} \sum_i  \ell(h(x_i),y_i) \mu(d h) & =  \int_{h\in \cH_{\tol+\Delta}} L_S(h) \mu(d h) \\
    &\leq \int_{h\in \cH_{\tol+\Delta}}\prn*{\min_{h\in \cH} L_S(h) +\tol+ \Delta} \mu(d h) \\
    &\leq \mu(\cH_{\tol+\Delta}) \prn*{ \min_{h\in \cH} L_S(h) +\tol+ \Delta}.
\end{align*}
Combine these two inequalities and \cref{ass:key}, we obtain 
\begin{align*}
    n\;\mathrm{LOO}_S(\set{\widehat y_{\tol,i}}_{i\in [n]}) &= \sum_{i} \ell(\yhat_{\tol,i},y_i) \\
    &\leq c\frac{\mu(\cH_{\tol+\Delta})}{\mu (\cH_{\tol-\Delta})} \prn*{ \min_{h\in \cH} L_S(h) +\tol+ \Delta} \leq  cC_g \prn*{\min_{h\in \cH} L_S(h) +\tol+ \Delta }.
\end{align*}
This concludes our proof.
\end{proof}

\subsection{Median Aggregation over Tolerance Levels}
\label{sec:median-aggregation-over-tol}

Choosing a single tolerance level $t$ is delicate in the LOO setting for two intertwined
reasons.
First, the local level-set growth behavior may vary across training samples, so a tolerance
that is well behaved for one subsample need not be appropriate for another.
Second—and more fundamentally—an LOO predictor for index $i$ is constructed without access
to $y_i$, even though the local growth condition is defined in terms of the full-sample
empirical risk and therefore depends on all responses.
As a result, no single data-dependent tolerance level can be reliably selected by all
leave-one-out predictors.
To overcome this difficulty, we introduce the second layer of aggregation by taking median over a set of tolerance levels in \cref{alg:levelset-aggregation}, which yields robustness to tolerance misspecification while remaining fully
compatible with the LOO constraint. The condition and guarantee are stated in \cref{ass:grid-key} and \cref{thm:main} respectively.

\begin{assumption}[Level-set growth on a grid]
\label{ass:grid-key}
Let $\rho>1/2$ and $\cT$ be a finite set of tolerances (often chosen to be a grid).
We say $(\cH,\ell)$ satisfies the level-set growth condition with parameters
$(\mu,\cT,\Delta,C_g,\rho)$ if
\[
\bigl|\{t\in\cT:(\cH,\ell)\text{ satisfies \cref{ass:key} at }t\}\bigr|
\;\ge\; \rho|\cT|.
\]
\end{assumption}

This assumption is typically mild.
If the local level-set growth condition \cref{ass:key} fails at a tolerance $t$,
the corresponding level set must expand by a factor of at least $C_g$.
For a finite hypothesis class (or one made finite by discretization), such
multiplicative expansions can occur only $O(\log|\cH|)$ times before the level
set saturates.
Consequently, for any grid $\cT$ with $|\cT| = C \log|\cH|$ and $C$ sufficiently
large, \cref{ass:key} can fail for only a small fraction of tolerances in $\cT$.

\begin{theorem}
\label{thm:main}
Assume the loss $\ell$ is either monotone in distance, i.e.,
$\ell(y'',y)\ge \ell(y',y)$ whenever $|y''-y|\ge |y'-y|$ or one-sided monotone in the first argument with known direction.
Suppose that \cref{ass:agg} holds and that \cref{ass:grid-key} holds with parameter $(\mu,\cT,\Delta,C_g,\rho)$. 
Let $\{\hat y_i\}_{i\in[n]}$ denote the output of
\cref{alg:levelset-aggregation}, and define
$t_{\max}:=\max_{t\in\cT} t$.
Then
\[
\mathrm{LOO}_S(\{\hat y_i\}_{i\in[n]})
\;\le\;
\frac{2cC_g}{(2\rho-1)n}
\bigl(\min_{h\in \cH} L_S(h)+t_{\max}+\Delta\bigr).
\]
\end{theorem}

\begin{proof}[\pfref{thm:main}]
For each data point $i$, let 
\[
\epsilon_i := \ell(\wh y_i , y_i),
\] 
Define the sets of \emph{good} and \emph{bad} levels:
\[
G := \{\tol \in \cT : (\cH,\ell) \text{ satisfies \cref{ass:key} with } (\mu, \tol, \Delta, C_g)\},\quad 
B := \cT \setminus G.
\]
By the level-set growth condition \cref{ass:grid-key}, we have $|G| \ge \rho |\cT|$ and $|B| \le (1-\rho)|\cT|$.  
By definition of the median, at least $|\cT|/2$ values
$\hat y_{\tol,i}$ satisfy $\hat y_{\tol,i} \le \hat y_i$,
and at least $|\cT|/2$ satisfy $\hat y_{\tol,i} \ge \hat y_i$.
Using monotonicity of $\ell$, for at least 
\[
|\cT|/2 - |B| \ge (\rho - 1/2) |\cT|
\] 
\emph{good} levels $\tol \in G$, we have 
\[
\ell(\hat{y}_{\tol,i}, y_i) \ge \epsilon_i.
\]
Summing over $i$ and good levels gives
\[
\sum_{i=1}^n \sum_{\tol \in G} \ell(\hat{y}_{\tol,i}, y_i) \ge (\rho - 1/2)|\cT| \sum_{i=1}^n \epsilon_i = (\rho - 1/2)|\cT| \sum_{i=1}^n \ell(\hat{y}_i, y_i).
\]
On the other hand, for any good level $\tol \in G$, by Proposition \ref{prop:good-level-guarantee} we have
\[
\sum_{i=1}^n \ell(\hat{y}_{\tol,i}, y_i) \le c C_g \bigl(\min_{h\in \cH} L_S(h) + \tol + \Delta \bigr) \le cC_g \bigl(\min_{h\in \cH} L_S(h)+ \maxtol + \Delta \bigr).
\]
Since there are at most $|\cT|$ levels,
\[
\sum_{i=1}^n \sum_{\tol \in G} \ell \bigl(\hat{y}_{\tol,i}, y_i) \le |\cT| \cdot cC_g (\min_{h\in \cH} L_S(h) + \maxtol + \Delta \bigr).
\]
Comparing the upper bound and lower bound gives
\[
(\rho - 1/2) |\cT| \sum_{i=1}^n \ell(\hat{y}_i, y_i) \le |\cT| \cdot cC_g \bigl(\min_{h\in \cH} L_S(h) + \maxtol + \Delta \bigr),
\]
which implies
\[
\sum_{i=1}^n \ell(\hat{y}_i, y_i) \le \frac{2cC_g}{2\rho - 1} \bigl(\min_{h\in \cH} L_S(h) + \maxtol + \Delta \bigr).
\]
hence
\[
\mathrm{LOO}_S(\{\hat y_i\}_{i\in [n]}) = \frac{1}{n} \sum_{i=1}^n \ell(\hat{y}_i, y_i) \le \frac{2cC_g}{(2\rho - 1)n} \bigl(\min_{h\in \cH} L_S(h) + \maxtol + \Delta \bigr). 
\]
This concludes our proof.
\end{proof}

Without identifying a single well-behaved tolerance level, the
Median of Level-Set Aggregation achieves a multiplicative oracle inequality for
the LOO error as long as a strict majority of tolerances in $\cT$ satisfy the local
growth condition, resolving the instability of data-dependent tolerance selection
in the leave-one-out setting.
The factor $2/(2\rho-1)$ quantifies the price of robustness to tolerance
misspecification, converting a strict majority of valid tolerances into a uniform
LOO guarantee. \looseness=-1

\section{Application to Classification with $0$-$1$ Loss}
\label{sec:classification}
We now instantiate the general framework of \cref{sec:overview} for binary classification.
We consider binary classification with $0$--$1$ loss.
Let $\cH\subseteq\{0,1\}^{\cX}$ have VC dimension $d$.
In the transductive setting, we may replace $\cH$ by its restriction to
$\{x_i\}_{i\in[n]}$, which is a finite class without changing any empirical or
leave-one-out risks.

Fix $i\in[n]$ and a tolerance level $t$.
We aggregate the leave-one-out level set $\cH_{t,i}$ by majority vote:
\[
\hat y_{t,i}
= \agg(\cH_{t,i},x_i)
= \mathbf{1}\Bigl\{\sum_{h\in\cH_{t,i}} (2h(x_i)-1)\ge 0\Bigr\}.
\]
Majority vote satisfies \cref{ass:agg} with $c=2$ for the $0$--$1$ loss under the counting
measure. We now verify \cref{ass:grid-key}.

\begin{lemma}
\label{lem:local-growth-for-classification}
Let $\cH\subseteq\{0,1\}^{\cX}$ have VC dimension $d$, and let $\ell_{0/1}$ be the
$0$--$1$ loss.
Let $\mu$ be the counting measure on $\cH$.
Let $\Delta = 1$ and $\cT = \{1,2,\dots,24 d \log n\}$.
Then $(\cH,\ell_{0/1})$ satisfies \cref{ass:grid-key} with parameters
\[
(\mu,\cT,\Delta,C_g,\rho)
=
(\text{counting},\ \cT,\ 1,\ 2,\ 3/4).
\]
\end{lemma}

\begin{proof}[\pfref{lem:local-growth-for-classification}]
We have $0\le \ell_{0/1}(y,y')\le 1$ for all $y,y'\in\cY$. By \Cref{lem:level-set-sandwich}, $\cH_{t-\Delta}\subseteq\cH_{t,i}\subseteq\cH_{t+\Delta}$ with $\Delta=1$ for all $t$ and $i$, so it suffices to
control the growth ratio $\mu(\cH_{\tol+1})/\mu(\cH_{\tol-1})$. Fix a tolerance level $\tol\in\cT$.
The local level-set growth condition \cref{ass:key} with $\Delta=1$ requires
\[
\frac{\mu(\cH_{\tol+1})}{\mu(\cH_{\tol-1})} \le C_g .
\]
We now show that this inequality holds with $C_g=2$ for all but at most
$6 d\log n$ values of $\tol$.
Define the set of bad levels
\[
B := \{\tol\in\cT : \mu(\cH_{\tol+1}) > 2\mu(\cH_{\tol-1})\},
\qquad s := |B|.
\]
Order $B=\{\tol_1<\tol_2<\dots<\tol_s\}$.
For each $\tol_j\in B$,
\[
\mu(\cH_{\tol_j+1}) > 2\mu(\cH_{\tol_j-1}).
\]
Consider the subsequence $\tol_1,\tol_3,\tol_5,\dots$.
Iterating the above inequality along this subsequence yields
\[
\mu(\cH_{\tol_{2\ell-1}+1}) \ge 2^{\ell}
\qquad
\text{for }\ell=1,\dots,\lceil s/2\rceil.
\]
In particular,
\[
\mu(\cH_{\tol_s+1}) \ge 2^{\lceil s/2\rceil}.
\]
Since $\cH$ is restricted to its projections on $\{x_i\}_{i=1}^n$ and has VC
dimension $d$, Sauer's lemma gives
\[
\mu(\cH_{\tol_s+1}) \le |\cH|
\le \sum_{k=0}^d \binom{n}{k}
\le \left(\frac{en}{d}\right)^d .
\]
Combining the lower and upper bounds,
\[
2^{s/2} \le \left(\frac{en}{d}\right)^d,
\]
which implies
\[
s \le 2\frac{d\ln(en/d)}{\ln 2}.
\]
For $n\ge3$, $\ln(en/d)\le2\ln n$, and therefore
\[
s \le 6 d\ln n.
\]
Since $|\cT|=24 d\ln n$, at least
\[
|\cT|-s \ge 18 d\ln n = \tfrac34 |\cT|
\]
levels satisfy
\[
\frac{\mu(\cH_{\tol+1})}{\mu(\cH_{\tol-1})}\le2.
\]
Thus $(\cH,\ell_{0/1})$ satisfies the level-set growth condition
\cref{ass:grid-key} with parameters $(\text{counting},\cT,1,2,3/4)$.
\end{proof}

Combining
\cref{lem:local-growth-for-classification} with \cref{thm:main} yields the
following corollary.

\begin{corollary}
\label{cor:loo-01}
Let $\cH\subseteq\{0,1\}^{\cX}$ have VC dimension $d\geq 1$, $n\geq 3$, and let $\ell_{0/1}$ be the
$0$--$1$ loss.
Let $\{\hat y_i\}_{i=1}^n$ be the output of
\cref{alg:levelset-aggregation} with tolerance grid
$\cT=\{1,\dots,24 d\log n\}$.
Then
\[
\mathrm{LOO}_S(\{\hat y_i\}_{i\in [n]})
\;\le\;
\frac{16}{n}\min_{h\in \cH} L_S(h) +\frac{400}{n} d\log n.
\]
\end{corollary}

\begin{proof}[\pfref{cor:loo-01}]
By Lemma~\ref{lem:local-growth-for-classification},
$(\cH,\ell_{0/1})$ satisfies the level-set growth condition
\cref{ass:grid-key} with parameters
\[
(\mu,\cT,\Delta,C_g,\rho) = (\text{counting},\cT,1,2,3/4).
\]
Applying Theorem~\ref{thm:main} yields
\[
\mathrm{LOO}_S(\{\hat y_i\}_{i\in [n]})
\le
\frac{4C_g}{(2\rho-1)n}
\bigl(\min_{h\in \cH} L_S(h)+\maxtol+\Delta\bigr).
\]
Substituting $C_g=2$, $\rho=3/4$, $\maxtol=24 d\log n$, and $\Delta=1$ gives
\begin{align*}
\mathrm{LOO}_S(\{\hat y_i\}_{i\in [n]})
&\le
\frac{16}{n}\bigl(\min_{h\in \cH} L_S(h)+24 d\log n+1\bigr) \\
& \le \frac{16}{n}\bigl(\min_{h\in \cH} L_S(h)+25 d\log n\bigr)
\end{align*}
This concludes our proof.
\end{proof}

Corollary~\ref{cor:loo-01} establishes a leave-one-out oracle inequality for
classification under the $0$--$1$ loss over \emph{arbitrary} VC classes.
In the realizable case, it yields $\mathrm{LOO}_S=O(d\log n/n)$, which is
near-optimal up to a $\log n$ factor compared to the optimal $O(d/n)$ rate obtained by \cite{haussler1994predicting}. In the agnostic case, similar multiplicative oracle inequality is obtained by \cite{long1998complexity} with 
$\mathrm{LOO}_S(\{\hat y_i\}_{i\in [n]})
\le
\frac{15}{n}\bigl(\min_{h\in \cH} L_S(h)+ d\bigr)$. We are also comparable up to a $\log n$ factor.
\looseness=-1

\section{Application to Regression with Convex Loss}
\label{sec:regression}
We instantiate the general framework of \cref{sec:overview} for regression.
We now consider real-valued regression under a bounded convex loss.
Let $\cH\subseteq \cY^{\cX}$ be a finite hypothesis class, and let
$\ell:\cY\times\cY\to\R_+$ be convex, monotone in distance, and uniformly bounded by $M$.
For each $i\in[n]$ and tolerance level $t$, we aggregate $\cH_{t,i}$ by averaging:
\[
\hat y_{t,i}
:= \agg(\cH_{t,i},x_i)
:= \frac{1}{|\cH_{t,i}|}\sum_{h\in\cH_{t,i}} h(x_i).
\]
By Jensen's inequality, averaging satisfies \cref{ass:agg} with $c=1$ for any convex loss.
We now verify \cref{ass:grid-key} in this setting.

\begin{lemma}
\label{lem:local-growth-for-bounded-convex}
Let $\cH$ be a finite hypothesis class, and let
$\ell:\cY\times\cY\to\R_+$ be convex, monotone in distance, and upper bounded by $M$.
Let $\mu$ be the counting measure on $\cH$.
Let $\Delta = M$ and $\cT := \{M,2M,3M,\dots,12M\log|\cH|\}$.
Then $(\cH,\ell)$ satisfies \cref{ass:grid-key} with parameters
\[
(\mu,\cT,\Delta,C_g,\rho)
=
(\text{counting},\ \cT,\ M,\ 2,\ 3/4).
\]
\end{lemma}

\begin{proof}[\pfref{lem:local-growth-for-bounded-convex}]
We have $0\le \ell(y,y')\le M$ for all $y,y'\in\cY$.
By \Cref{lem:level-set-sandwich}, for every $i\in[n]$ and every $t\ge0$,
\[
\cH_{t-\Delta}
\;\subseteq\;
\cH_{t,i}
\;\subseteq\;
\cH_{t+\Delta}
\qquad
\text{with }\Delta=M.
\]
so it suffices to
control the growth ratio $\mu(\cH_{\tol+M})/\mu(\cH_{\tol-M})$. Fix a tolerance level $t\in\cT$.
The local level-set growth condition \cref{ass:key} with $\Delta=M$
requires
\[
\frac{\mu(\cH_{t+M})}{\mu(\cH_{t-M})} \le C_g .
\]
We show that this inequality holds with $C_g=2$ for all but at most
$3\log|\cH|$ values of $t$.
Define the set of bad levels
\[
B := \{t\in\cT : \mu(\cH_{t+M}) > 2\mu(\cH_{t-M})\},
\qquad
s := |B|.
\]
Order $B=\{t_1<t_2<\dots<t_s\}$.
For each $t_j\in B$,
\[
\mu(\cH_{t_j+M}) > 2\mu(\cH_{t_j-M}).
\]
Consider the subsequence $t_1,t_3,t_5,\dots$.
Iterating the inequality along this subsequence yields
\[
\mu(\cH_{t_{2\ell-1}+M})
\;\ge\;
2^\ell
\qquad
\text{for }\ell=1,\dots,\lceil s/2\rceil.
\]
In particular,
\[
\mu(\cH_{t_s+M}) \ge 2^{\lceil s/2\rceil}.
\]
Since $\cH_{t_s+M}\subseteq\cH$ and $\mu$ is the counting measure,
\[
\mu(\cH_{t_s+M}) \le |\cH|.
\]
Combining the bounds,
\[
2^{s/2} \le 2^{\lceil s/2\rceil} \le |\cH|,
\]
which implies
\[
s \le \frac{2\ln|\cH|}{\ln 2} \le 3\ln|\cH|.
\]
Since $|\cT| = 12\ln|\cH|$, at least
\[
|\cT|-s \ge 9\ln|\cH| = \tfrac34|\cT|
\]
levels satisfy
\[
\frac{\mu(\cH_{t+M})}{\mu(\cH_{t-M})} \le 2.
\]
Thus $(\cH,\ell)$ satisfies the level-set growth condition
\cref{ass:grid-key} with parameters
$(\text{counting},\cT,M,2,3/4)$.
\end{proof}

The proof is similar to the proof of \cref{lem:local-growth-for-classification}. Combining
\cref{lem:local-growth-for-bounded-convex} with \cref{thm:main} yields:

\begin{corollary}
\label{cor:loo-bounded-convex}
Let $\cH$ be a finite hypothesis class, and let
$\ell:\cY\times\cY\to\R_+$ be convex, monotone in distance, and upper bounded by $M$.
Let $\{\hat y_i\}_{i=1}^n$ be the output of
\cref{alg:levelset-aggregation} with tolerance grid $\cT=\{M,2M,3M,\dots,12M\log|\cH|\}$.
Then
\[
\mathrm{LOO}_S(\{\hat y_i\}_{i\in [n]})
\;\le\; 
\frac{8}{n}\min_{h\in \cH} L_S(h)
\;+\;
\frac{104}{n}M\log|\cH|.
\]
\end{corollary}

\begin{proof}[\pfref{cor:loo-bounded-convex}]
By Lemma~\ref{lem:local-growth-for-bounded-convex},
$(\cH,\ell)$ satisfies the level-set growth condition
\cref{ass:grid-key} with parameters
\[
(\mu,\cT,\Delta,C_g,\rho)
=
(\text{counting},\cT,M,2,3/4).
\]
Applying Theorem~\ref{thm:main} yields
\[
\mathrm{LOO}_S(\{\hat y_i\}_{i\in [n]})
\le
\frac{2C_g}{(2\rho-1)n}
\bigl(\min_{h\in \cH} L_S(h)+\maxtol+\Delta\bigr).
\]
Substituting $C_g=2$, $\rho=3/4$,
$\maxtol=12M\log|\cH|$, and $\Delta=M$ gives
\begin{align*}
\mathrm{LOO}_S(\{\hat y_i\}_{i\in [n]})
&\le
\frac{8}{n}
\bigl(
\min_{h\in \cH} L_S(h)
+ 12M\log|\cH|
+ M
\bigr). \\
&\le\frac{8}{n}
\bigl(
\min_{h\in \cH} L_S(h)
+ 13M\log|\cH|
\bigr).
\end{align*}
This concludes our proof.
\end{proof}

Corollary~\ref{cor:loo-bounded-convex} establishes a leave-one-out oracle inequality
for regression with bounded convex loss over \emph{arbitrary} finite hypothesis
classes.
Previously, comparable LOO guarantees were only available in specialized linear
settings. For instance, \citet{forster2002relative} show
$\mathrm{LOO}_S \le \tfrac{1}{n}\min_{h\in\cH}L_S(h)+\tfrac{2Md}{n}$.
In \cref{app:loo-for-vovk}, we establish a leave-one-out bound of $\mathrm{LOO}_S \le \tfrac{2}{n}\min_{h\in\cH}L_S(h)+\tfrac{2Md}{n}$ for a linear
predictor motivated by the Vovk–Azoury–Warmuth framework \citep{vovk2001competitive,azoury2001relative}.
Unlike the classical VAW predictor, which is defined online using prefix Gram
matrices, our analysis uses the full-sample Gram matrix, which is natural in the
transductive setting.
Moreover, stability-based analyses such as \citet{zhang2003leave} derive expected
LOO bounds for specific regularized algorithms in RKHS.
In contrast, Corollary~\ref{cor:loo-bounded-convex} removes the reliance on linear or
Hilbert-space structure and applies to any convex loss monotone in distance. \looseness=-1

\section{Application to Density Estimation with Log Loss}
\label{sec:density-estimation}

We instantiate the framework of \cref{sec:overview} in the setting of density estimation under
log loss.
Let $\cP$ be a finite class of probability densities on $\cX$ with respect to a common dominating
measure, and let $S=\{x_1,\dots,x_n\}$ be an arbitrary individual sequence.
We take the hypothesis class to be $\cH=\cP$ and consider the log loss
\[
\ell_{\log}(p;x) := -\log p(x),
\]
which depends only on the prediction $p$ and the covariate $x$ (there is no response variable).
The log loss is the canonical choice for density estimation, as it corresponds to the negative
log-likelihood and is related to coding \citep{mourtada2022improper,cover1999elements}.

Under this choice of loss, the full-sample and leave-one-out empirical losses are
\[
L_S(p) := \sum_{j=1}^n -\log p(x_j),
\qquad
L_{S_{-i}}(p) := \sum_{j\neq i} -\log p(x_j),
\]
and the corresponding level sets $\cP_t$ and $\cP_{t,i}$ are defined exactly as in
\cref{sec:overview}.
For a tolerance level $t$, the level-set aggregation rule averages densities over $\cP_{t,i}$,
yielding the pointwise predictor
\[
\hat p_{t,i}
:= \agg(\cP_{t,i},x_i)
:= \frac{1}{|\cP_{t,i}|}\sum_{p\in \cP_{t,i}} p(x_i).
\]
The final prediction $\hat p_i$ is obtained via \cref{alg:levelset-aggregation}.
By Jensen’s inequality, simple averaging satisfies the aggregation condition
\cref{ass:agg} for the log loss $\ell_{\log}$.

We now verify \cref{ass:grid-key} in this setting under a mild bounded log-density ratio
assumption.

\begin{lemma}
\label{lem:density-level-set-sandwich}
Assume there exists $M>0$ such that for all $p,q\in\cP$ and all $x\in\cX$, $\bigl| \log\frac{p(x)}{q(x)} \bigr| \le M$.
Then for every $i\in[n]$ and every $t\ge0$, $\cP_{t-M}
\;\subseteq\;
\cP_{t,i}
\;\subseteq\;
\cP_{t+M}.$
\end{lemma}

\begin{proof}[\pfref{lem:density-level-set-sandwich}]
Define $p^\star\in\arg\min_{p\in \cP}L_S(p)$ and $p_{-i}^\star\in\arg\min_{p\in\cP}L_{S_{-i}}(p)$. We first show that $\mathcal P_{t-M} \subseteq \mathcal P_{t,i}$.  
Let $p \in \mathcal P_{t-M}$, so that
\[
L_S(p) \le L_S(p^\star) + t - M.
\]
Then
\begin{align*}
L_{S_{-i}}(p)
&= L_S(p) + \log p(x_i) \\
&\le L_S(p^\star) + t - M + \log p(x_i).
\end{align*}
Since $p^\star$ minimizes $L_S$,
\[
L_S(p^\star) \le L_{S_{-i}}(p^\star_{-i}) - \log p^\star_{-i}(x_i),
\]
we have
\begin{align*}
L_{S_{-i}}(p)
&\le L_{S_{-i}}(p^\star_{-i}) - \log p^\star_{-i}(x_i) + t - M + \log p(x_i) \\
&= L_{S_{-i}}(p^\star_{-i}) + t + \bigl(\log p(x_i)-\log p^\star_{-i}(x_i)\bigr) - M.
\end{align*}
By the uniform log-ratio assumption, $\log p(x_i) - \log p^\star_{-i}(x_i) \le M$, so
\[
L_{S_{-i}}(p) \le L_{S_{-i}}(p^\star_{-i}) + t,
\]
hence $p \in \mathcal P_{t,i}$.
Conversely, let $p \in \mathcal P_{t,i}$, so that
\[
L_{S_{-i}}(p) \le L_{S_{-i}}(p^\star_{-i}) + t.
\]
Then
\[
L_S(p) = L_{S_{-i}}(p) - \log p(x_i) \le L_{S_{-i}}(p^\star_{-i}) + t - \log p(x_i).
\]
Using the fact that $p^\star_{-i}$ minimizes $L_{S_{-i}}$, we have
\[
L_{S_{-i}}(p^\star_{-i}) \le L_S(p^\star) + \log p^\star(x_i),
\]
so
\[
L_S(p) \le L_S(p^\star) + t - \log p(x_i) + \log p^\star(x_i).
\]
By the uniform log-ratio assumption, $\log p^\star(x_i) - \log p(x_i) \le M$, hence
\[
L_S(p) \le L_S(p^\star) + t + M,
\]
so $p \in \mathcal P_{t+M}$. This concludes our proof.
\end{proof}

\begin{lemma}
\label{lem:local-growth-for-density}
Let $\cP$ be a finite class of probability densities satisfying $|\log \frac{p(x)}{q(x)}| \le M$ for all $p,q\in\cP,\ x\in\cX$.
Let $\mu$ be the counting measure on $\cP$, $\Delta=M$, and $\cT := \{M,2M,3M,\dots,12M\log|\cP|\}$.
Then $(\cP,\ell_{\log})$ satisfies the level-set growth condition
\cref{ass:grid-key} with parameters
\[
(\mu,\cT,\Delta,C_g,\rho)
=
(\text{counting},\cT,M,2,3/4).
\]
\end{lemma}

\begin{proof}[\pfref{lem:local-growth-for-density}]
By \Cref{lem:density-level-set-sandwich},
for every $i\in[n]$ and every $t\ge0$,
\[
\mathcal P_{t-M}
\;\subseteq\;
\mathcal P_{t,i}
\;\subseteq\;
\mathcal P_{t+M},
\]
so it suffices to
control the growth ratio $\mu(\cP_{\tol+M})/\mu(\cP_{\tol-M})$.
Fix a tolerance level $t\in\cT$. The local level-set growth condition \cref{ass:key} with $\Delta=M$
requires
\[
\frac{\mu(\mathcal P_{t+M})}{\mu(\mathcal P_{t-M})} \le C_g .
\]
We now show that this inequality holds with $C_g=2$ for all but at most
$3\log|\mathcal P|$ values of $t$.
Define the set of bad levels
\[
B
:=
\{t\in\cT : \mu(\mathcal P_{t+M}) > 2\mu(\mathcal P_{t-M})\},
\qquad
s := |B|.
\]
Order $B=\{t_1<t_2<\dots<t_s\}$.
For each $t_j\in B$,
\[
\mu(\mathcal P_{t_j+M}) > 2\mu(\mathcal P_{t_j-M}).
\]
Consider the subsequence $t_1,t_3,t_5,\dots$.
Iterating along this subsequence yields
\[
\mu(\mathcal P_{t_{2\ell-1}+M})
\;\ge\;
2^\ell
\qquad
\text{for }\ell=1,\dots,\lceil s/2\rceil.
\]
In particular,
\[
\mu(\mathcal P_{t_s+M}) \ge 2^{\lceil s/2\rceil}.
\]
Since $\mathcal P_{t_s+M}\subseteq\mathcal P$ and $\mu$ is the counting measure,
\[
\mu(\mathcal P_{t_s+M}) \le |\mathcal P|.
\]
Therefore,
\[
2^{s/2} \le |\mathcal P|,
\]
which implies
\[
s \le \frac{2\ln|\mathcal P|}{\ln 2}
\le 3\ln|\mathcal P|.
\]
Since $|\cT| = 12\ln|\mathcal P|$, at least
\[
|\cT| - s \ge 9\ln|\mathcal P| = \tfrac34|\cT|
\]
levels satisfy
\[
\frac{\mu(\mathcal P_{t+M})}{\mu(\mathcal P_{t-M})} \le 2.
\]
Thus $(\mathcal P,\ell_{\log})$ satisfies the level-set growth condition
\cref{ass:grid-key} with parameters $(\text{counting},\cT,M,2,3/4)$.
\end{proof}

The proof is similar to the proof of \cref{lem:local-growth-for-classification}. Combining \cref{lem:local-growth-for-density} with \cref{thm:main} yields:

\begin{corollary}
\label{cor:loo-density}
Let $\cP$ be a finite class of probability densities satisfying $|\cP|\geq 3$ and $|\log \frac{p(x)}{q(x)}| \le M$ for all $p,q\in\cP,\ x\in\cX $.
Let $\{\hat p_i =\med(\{\hat p_{t,i}\}_{t\in\cT}) \}_{i=1}^n$ be the output of
\cref{alg:levelset-aggregation} with tolerance grid $\cT=\{M,2M,\dots,12M\log|\cP|\}$.
Then
\[
\mathrm{LOO}_S(\{\hat p_i\}_{i\in [n]})
\;\le\;
\frac{8}{n}
\min_{p\in \cP}L_S(p)
+ \frac{104}{n}M\log|\cP|.
\]
\end{corollary}

\begin{proof}[\pfref{cor:loo-density}]
By Lemma~\ref{lem:local-growth-for-density},
$(\mathcal P,\ell_{\log})$ satisfies the level-set growth condition
\cref{ass:grid-key} with parameters
\[
(\mu,\cT,\Delta,C_g,\rho)
=
(\text{counting},\cT,M,2,3/4).
\]
Applying Theorem~\ref{thm:main} yields
\[
\mathrm{LOO}_S(\{\hat p_i\}_{i\in [n]})
\le
\frac{2C_g}{(2\rho-1)n}
\bigl(\min_{p\in \cP}L_S(p)+\maxtol+\Delta\bigr).
\]
Substituting $C_g=2$, $\rho=3/4$,
$\maxtol=12M\log|\mathcal P|$, and $\Delta=M$ gives
\begin{align*}
\mathrm{LOO}_S(\{\hat p_i\}_{i\in [n]})
&\le
\frac{8}{n}
\Bigl(
\min_{p\in \cP}L_S(p)
+ 12M\log|\mathcal P|
+ M
\Bigr)  \\
& \le \frac{8}{n}
\Bigl(
\min_{p\in \cP}L_S(p)
+ 13M\log|\mathcal P|
\Bigr)
\end{align*}
\end{proof}

Corollary~\ref{cor:loo-density} establishes a leave-one-out oracle inequality for density
estimation over \emph{arbitrary} finite classes $\cP$, without requiring stability.
Previously, comparable LOO guarantees were only known in specialized settings.
For instance, in the Bernoulli case, \citet{forster2002relative} obtain the sharper bound
$\mathrm{LOO}_S \le \tfrac{1}{n}\min_{p\in\cP}L_S(p)+\tfrac{1}{n}$ via problem-specific arguments.
More recently, \citet{mourtada2022improper} derive LOO bounds for general (possibly infinite)
density classes, but their analysis relies on stability conditions not satisfied by arbitrary
finite classes.
In contrast, Corollary~\ref{cor:loo-density} applies uniformly to any finite $\cP$ with no
structural assumptions beyond finiteness; the bounded log-density ratio condition is used
only to verify level-set growth and can be removed by smoothing
(\cref{sec:smoothing}).

\subsection{Removing Boundedness Assumptions by Smoothing}
\label{sec:smoothing}
The bounded log-density ratio condition can be enforced by a standard
smoothing argument. Here we introduce two types of smoothing, one by averaging over the probability density class, another with uniform distribution on the space $\cX$, when it is finite.

\begin{lemma}
\label{lem:loo-density-smooth}
Let $\cP$ be a finite class of probability densities over $\cX$. Let $\bar p:= \tfrac{1}{|\cP|}\sum_{p\in\cP} p$ and $U := \mathrm{Unif}(\cX)$ be the uniform distribution on $\cX$ when it is finite .
For $\eps\in(0,1/2)$, define the smoothed class $\cP'$ by
$\cP' := \{(1-\eps)p + \eps \nu : p\in\cP\}$, where $\nu=\bar p$ if $|\cX|\ge|\cP|$ and $\nu=U$ otherwise.
Let $\{\hat p'_i\}_{i=1}^n$ be the predictors produced by
\cref{alg:levelset-aggregation} applied to $\cP'$.
Then
\[
\mathrm{LOO}_S(\{\hat p'_i\}_{i\in [n]})
\;\le\;
\frac{8}{n}
\Bigl(
\min_{p\in \cP} L_S(p)
+
13 M_\eps \log|\cP|
\Bigr)
+
16\eps,
\]
where $M_\eps = \log\tfrac{1}{\eps} + \min(\log|\cP|,\log|\cX|)$.
In particular, choosing $\eps = 1/n$ yields
\[
\mathrm{LOO}_S(\{\hat p'_i\}_{i\in [n]})
\le
\frac{8}{n}\min_{p\in\cP}L_S(p)
+
\frac{112}{n}\log|\cP|\cdot \min(\log|\cP|,\log|\cX|)
+
\frac{112}{n}\log|\cP|\cdot\log n
\]
\end{lemma}

\begin{proof}[\pfref{lem:loo-density-smooth}]
We first consider the case $|\cX|\ge|\cP|$, where 
\[\cP' := \{(1-\eps)p + \eps\bar p : p\in\cP\}\]
and
$\bar p := \tfrac{1}{|\cP|}\sum_{p\in\cP} p$.
Fix $p',q'\in\cP'$. By definition,
\[
p'=(1-\eps)p+\eps\bar p,
\qquad
q'=(1-\eps)q+\eps\bar p
\]
for some $p,q\in\cP$.
Since $\bar p(x)\ge \tfrac{1}{|\cP|}p(x)$ for all $x$, we have
\[
p'(x) \le (1-\eps)|\cP| \cdot \bar p(x) + \eps\bar p(x)
\le |\cP|\cdot\bar p(x),
\qquad
q'(x)\ge \eps\bar p(x).
\]
Therefore,
\[
\Bigl|\log\frac{p'(x)}{q'(x)}\Bigr|
\le
\log|\cP| + \log\tfrac{1}{\eps}
=: M_\eps .
\]
Since $(\cP',\ell_{\log})$ satisfies the bounded log-density ratio condition
with constant $M_\eps$, by \cref{lem:local-growth-for-density}, it satisfies the level-set growth condition
\cref{ass:grid-key} with parameters $(\text{counting},\cT,M_\eps,2,3/4)$. Applying \cref{thm:main}
\[
\mathrm{LOO}_S(\{\hat p'_i\})
\le
\frac{8}{n}
\Bigl(
L_S(p^{\prime\star})
+
13 M_\eps \log|\mathcal P|
\Bigr),
\]
where $p^{\prime\star}\in\arg\min_{p'\in\cP'} L_S(p')$.
Let
\[
p^\star := \arg\min_{p\in\cP} L_S(p),
\qquad
p^\star_\eps := (1-\eps)p^\star + \eps\bar p \in \cP'.
\]
For each $i$,
\[
p^\star_\eps(x_i)\ge (1-\eps)p^\star(x_i),
\]
hence
\[
-\log p^\star_\eps(x_i)
\le
-\log p^\star(x_i) - \log(1-\eps).
\]
Summing over $i$ and using $-\log(1-\eps)\le 2\eps$ for $\eps\in(0,1/2)$ gives
\[
L_S(p^\star_\eps)
\le
L_S(p^\star) + 2n\eps.
\]
By optimality of $p^{\prime\star}$,
\[
L_S(p^{\prime\star}) \le L_S(p^\star_\eps)
\le L_S(p^\star) + 2n\eps.
\]
Substituting into the LOO bound yields
\[
\mathrm{LOO}_S(\{\hat p_i'\}_{i\in [n]})
\le
\frac{8}{n}
\Bigl(
L_S(p^\star)
+
13 M_\eps \log|\mathcal P|
\Bigr)
+
16\eps,
\]
Plugging the value of $ M_\eps=\log|\cP| + \log n$
\begin{align*}
\mathrm{LOO}_S(\{\hat p_i'\}_{i\in [n]}) &\le \frac{8}{n} \min_{p\in \cP}L_S(p)
+
\frac{104}{n}\log^2|\mathcal P|
+
\frac{104}{n}(\log n)(\log|\mathcal P|)
+
\frac{16}{n} \\
& \le \frac{8}{n} \min_{p\in \cP}L_S(p)
+
\frac{112}{n}\log^2|\mathcal P|
+
\frac{112}{n}\log|\mathcal P|\cdot \log n
\end{align*}
We next consider the case $|\cX|<|\cP|$, where 
\[\cP' := \{(1-\eps)p + \eps U : p\in\cP\}\]
and $U := \mathrm{Unif}(\cX)$.
Fix $p',q'\in\cP'$. By definition,
\[
p'=(1-\eps)p+\eps U,
\qquad
q'=(1-\eps)q+\eps U 
\]
for some $p,q\in\cP$.
Since $\mathcal X$ is finite and $U$ is uniform,
\[
U(x)=\frac{1}{|\mathcal X|}
\quad\text{for all }x\in\mathcal X.
\]
Therefore, for all $x$,
\[
p'(x)\le 1,
\qquad
q'(x)\ge \eps U(x)=\frac{\eps}{|\mathcal X|}.
\]
Hence,
\[
\Bigl|\log\frac{p'(x)}{q'(x)}\Bigr|
\le
\log|\mathcal X| + \log\tfrac{1}{\eps}
=: M_\eps.
\]
Since $(\cP',\ell_{\log})$ satisfies the bounded log-density ratio condition
with constant $M_\eps$, by \cref{lem:local-growth-for-density}
it satisfies the level-set growth condition
\cref{ass:grid-key} with parameters $(\text{counting},\cT,M_\eps,2,3/4)$.
Applying \cref{thm:main} yields
\[
\mathrm{LOO}_S(\{\hat p_i'\}_{i\in [n]})
\le
\frac{8}{n}
\Bigl(
L_S(p^{\prime\star})
+
13 M_\eps \log|\cP|
\Bigr),
\]
where $p^{\prime\star}\in\arg\min_{p'\in\cP'} L_S(p')$.
Let
\[
p^\star := \arg\min_{p\in\cP} L_S(p),
\qquad
p^\star_\eps := (1-\eps)p^\star + \eps U \in \cP'.
\]
For each $i$,
\[
p^\star_\eps(x_i)\ge (1-\eps)p^\star(x_i),
\]
hence
\[
-\log p^\star_\eps(x_i)
\le
-\log p^\star(x_i) - \log(1-\eps).
\]
Summing over $i$ and using $-\log(1-\eps)\le 2\eps$ for $\eps\in(0,1/2)$ gives
\[
L_S(p^\star_\eps)
\le
L_S(p^\star) + 2n\eps.
\]
By optimality of $p^{\prime\star}$,
\[
L_S(p^{\prime\star}) \le L_S(p^\star_\eps)
\le L_S(p^\star) + 2n\eps.
\]
Substituting into the LOO bound yields
\[
\mathrm{LOO}_S(\{\hat p_i'\}_{i\in [n]})
\le
\frac{8}{n}
\Bigl(
L_S(p^\star)
+
13 M_\eps \log|\cP|
\Bigr)
+
16\eps,
\]
Plugging the value of $ M_\eps=\log|\mathcal X| + \log n$
\begin{align*}
\mathrm{LOO}_S(\{\hat p_i'\}_{i\in [n]}) &\le \frac{8}{n} \min_{p\in \cP}L_S(p)
+
\frac{104}{n}(\log|\mathcal X|+\log n)\log|\cP| + \frac{16}{n}  \\
& \le \frac{8}{n} \min_{p\in \cP}L_S(p)
+
\frac{112}{n}\log|\mathcal X|\cdot\log|\cP|
+
\frac{112}{n}\log|\cP|\cdot\log n
\end{align*}
Combining this bound with the corresponding expression for the case $|\cX|\ge|\cP|$, we conclude that, in general,
\[
\mathrm{LOO}_S(\{\hat p'_i\}_{i\in [n]})
\le
\frac{8}{n}\min_{p\in\cP}L_S(p)
+
\frac{112}{n}\log|\cP|\cdot \min(\log|\cP|,\log|\cX|)
+
\frac{112}{n}\log|\cP|\cdot\log n
\]
which completes the proof.
\end{proof}

\newcommand{\vol}{\mathrm{vol}}
\newcommand{\logit}{\mathrm{logit}}
\newcommand{\cE}{\mathcal{E}}

\section{Application to Logistic Regression}
\label{sec:logistic}

We instantiate the template of \cref{sec:overview} for logistic regression with bounded
covariates.
Let $R$ and $r$ be two positive numbers.
Given an arbitrary individual sequence
$S=\{(x_i,y_i)\}_{i=1}^n$ with $x_i\in\R^d$, $y_i\in\{\pm1\}$, and $\|x_i\|_2\le R$,
we consider the parameter class
$\cH=\{\theta\in\R^d:\|\theta\|_2\le r\}$.
For $\theta\in\cH$, the model predicts
\[
p_\theta(y\mid x) \;:=\; \sigma\!\bigl(y\,x^\top\theta\bigr),
\qquad \sigma(z):=\frac{1}{1+e^{-z}},
\]
and we evaluate predictions using the logistic loss
\[
\ell_{\logit}\bigl(p_\theta(\cdot\mid x);y\bigr)\;:=\;-\log p_\theta(y\mid x)
\;=\; -\log\sigma\!\bigl(y\,x^\top\theta\bigr).
\]
Let $A:=\sum_{i=1}^n x_i x_i^\top$ be the empirical Gram matrix. We assume $A$ to be non-degenerate with $\lambda_{\min}(A)>0$.
We define $\cH_A
:=
\left\{
\vartheta \in \mathbb{R}^d :
\inf_{\theta \in \cH}
\|\vartheta - \theta\|_A^2
\le rR
\right\}$ with $\|u\|_A^2 := u^\top A u$.
For each tolerance level $t\ge0$, we define the associated full-sample and leave-one-out
level sets 
\begin{equation}
\label{eq:logistic-level-sets}
\cH_t
:= \{\theta \in \cH_A : L_S(\theta)\le L_S(\theta^\star)+t\},
\qquad
\cH_{t,i}
:= \{\theta \in \cH_A: L_{S_{-i}}(\theta)\le L_{S_{-i}}(\theta_{-i}^\star)+t\},
\end{equation}
where $\theta^\star\in\arg\min_{\theta\in\cH}L_S(\theta)$ and $\theta_{-i}^\star\in\arg\min_{\theta\in\cH}L_{S_{-i}}(\theta)$.  We fix a single choice of $\theta^\star$ throughout this section. 
Note that the sets $\cH_t$ and $\cH_{t,i}$ are not required to be subsets of $\cH$; this relaxation is intentional and enables the volumetric arguments in
\cref{lem:contain-ellipsoid-logistic,lem:logistic-volume-ratio}. 
The level-set aggregation rule averages densities over $\cH_{t,i}$ with respect to a reference measure $\mu_B$ (defined below), i.e.,
\begin{equation}
\label{eq:logistic-agg}
\hat p_{t,i}
:= \frac{1}{\mu_B(\cH_{t,i})}\int_{\cH_{t,i}} \sigma(x_i^\top \theta) \,\mu_B(d\theta),
\end{equation}
and outputs
$\hat p_i$ via \cref{alg:levelset-aggregation}. 
Since the logistic is convex, 
Jensen's inequality implies that this aggregation rule
satisfies \cref{ass:agg}, and therefore fits directly into the level-set aggregation
framework.

We analyze the geometry of logistic level sets, relate them to ellipsoids defined by the empirical covariance, and verify the corresponding level-set growth conditions needed to apply the general LOO bound. 
We begin by relating logistic level sets to ellipsoids induced by the empirical
covariance matrix $A$. 
The next lemma shows that each logistic level set contains an explicit quadratic
neighborhood around the empirical risk minimizer, measured in the geometry defined by
$A$.
\begin{lemma}
\label{lem:contain-ellipsoid-logistic}
Assume that the empirical Gram matrix $A$
is non-degenerate with $\lambda_{\min}(A)>0$.
For any $t\ge0$, define the ellipsoid
\[
\mathcal{E}_{t}
:=
\Bigl\{
\theta\in\R^d :
(\theta-\theta^\star)^\top A (\theta-\theta^\star)\le t
\Bigr\},
\]
and its truncated version
\[
\mathcal{E}_t^{<}
:=
\mathcal{E}_t
\cap
\Bigl\{
\theta\in\R^d :
\nabla L_S(\theta^\star)^\top(\theta-\theta^\star)
\le 0
\Bigr\}.
\]
The truncation enforces a nonpositive first-order change in the objective at $\theta^\star$,
so $\mathcal{E}_t^{<}$ lies in the corresponding empirical-risk level set. Moreover, the
bounding hyperplane passes through $\theta^\star$, and thus the defining halfspace contains
at least half of the ellipsoid volume. Consequently,
\[
\mathcal{E}_{rR}^{<} \subseteq \cH_{rR},
\qquad
\vol(\mathcal{E}_{rR}^{<})
\ge \tfrac{1}{2}\,\vol(\mathcal{E}_{rR}).
\]
\end{lemma}

\begin{proof}[\pfref{lem:contain-ellipsoid-logistic}]
Let $z_i(\theta) := y_i x_i^\top \theta$. For logistic loss,
\[
\nabla^2 L_S(\theta)
=
\sum_{i=1}^n
\sigma(z_i(\theta))\sigma(-z_i(\theta))
x_i x_i^\top.
\]
Since $\sigma(z)\sigma(-z)\le\frac14$, we have
\[
\nabla^2 L_S(\theta)
\preceq
\frac14 \sum_{i=1}^n x_i x_i^\top.
\]
By Taylor's theorem with integral remainder,
\begin{align*}
L_S(\theta)
&=
L_S(\theta^\star)
+
\nabla L_S(\theta^\star)^\top(\theta-\theta^\star) \\
&\quad+
\int_0^1 (1-s)
(\theta-\theta^\star)^\top
\nabla^2 L_S(\theta^\star+s(\theta-\theta^\star))
(\theta-\theta^\star)\,ds.
\end{align*}
Using the Hessian bound,
\[
L_S(\theta)-L_S(\theta^\star)
\le
\nabla L_S(\theta^\star)^\top(\theta-\theta^\star)
+
\frac18
(\theta-\theta^\star)^\top
\sum_{i=1}^n x_i x_i^\top
(\theta-\theta^\star).
\]
Let $\theta\in\mathcal{E}_{rR}^{<}$.
Then
\[
\nabla L_S(\theta^\star)^\top(\theta-\theta^\star)\le 0,
\qquad
(\theta-\theta^\star)^\top \sum_{i=1}^n x_i x_i^\top(\theta-\theta^\star)\le rR.
\]
Substituting into the expansion above
\[
L_S(\theta)-L_S(\theta^\star)
\le
\frac{1}{8} rR
\le rR.
\]
Together with $\mathcal{E}_{rR}^{<} \subseteq \mathcal E_{rR} \subseteq \cH_A$, we show $\theta\in\cH_{rR}$, proving
\[
\mathcal{E}_{rR}^{<} \subseteq \cH_{rR}.
\]
Next we move to the Volume comparison. For $\forall t \ge0$, suppose $\theta^\star$ is interior of $\cH$, this implies
$\nabla L_S(\theta^\star)=0$. Hence  $\mathcal{E}_t^{<}=\mathcal{E}_t$ and $\vol(\cE_t^<)
= \vol(\cE_t)$. Suppose $\theta^\star$ lies at the boundary of $\cH$. Let
\[
u=A^{1/2}(\theta-\theta^\star).
\]
This is an invertible linear transformation with constant Jacobian
$|\det(A^{1/2})|$, so volume ratios are preserved. Then
\[
\cE_t
\longleftrightarrow
\{u:\|u\|_2^2\le t\},
\]
a Euclidean ball centered at the origin.
The cutting condition becomes
\[
\nabla L_S(\theta^\star)^\top(\theta-\theta^\star)\le 0
\iff
(A^{-1/2}\nabla L_S(\theta^\star))^\top u \le 0.
\]
Thus the truncation corresponds to intersecting the ball
with a half-space whose boundary hyperplane passes through the origin.
Hence
\[
\vol(\cE_t^<)
=
\frac{1}{2}
 \vol(\cE_t).\]
This completes the proof.
\end{proof}

This characterization motivates working with ellipsoidal reference measures.
Accordingly, we fix an ellipsoid
$B=\{\theta:\|A^{1/2}\theta\|_2\le R_B\}$ with radius
$R_B=\sqrt{n}\,rR+\sqrt{rR}$, which is large enough to contain the level set
$\cE_{rR}$ by $\norm{A^{1/2}\theta} \leq \norm{A^{1/2}\theta^\star} + \norm{A^{1/2}(\theta-\theta^\star)} \leq  \sqrt{n}\,rR+\sqrt{rR}$.
Let $\mu_B$ be the uniform distribution on $B$. Since $A$ is non-degenerate, $\mu_B$ is well-defined.
The next lemma upper bounds the global growth of the relevant level set in terms of the
probability measure $\mu_B$.

\begin{lemma}
\label{lem:logistic-volume-ratio}
Let $\cH_{rR}$, $B$ and $\mu_B$ be defined as above. Then we have
\[
\log\frac{1}{\mu_B(\cH_{rR})}
\le
d\log(8\vee 2nrR).
\]
\end{lemma}

\begin{proof}[\pfref{lem:logistic-volume-ratio}]
We first prove $\mathcal E_{rR} \subseteq B$.
Let $\theta \in \cE_{rR}$. Writing $\theta = \theta^\star + (\theta-\theta^\star)$, we have
\[
\|A^{1/2}\theta\|_2^2
= (\theta-\theta^\star)^\top A(\theta-\theta^\star)
+ 2(\theta-\theta^\star)^\top A\theta^\star
+ (\theta^\star)^\top A\theta^\star .
\]
By definition of $\cE_{rR}$, the first term is at most $rR$. The cross term is bounded as
\[
2(\theta-\theta^\star)^\top A\theta^\star
\le 2\sqrt{(\theta-\theta^\star)^\top A(\theta-\theta^\star)}
   \sqrt{(\theta^\star)^\top A\theta^\star}
\le 2\sqrt{rR}\,\sqrt{(\theta^\star)^\top A\theta^\star}.
\]
Since $\|x_i\|_2 \le R$ and $\|\theta^\star\|_2 \le r$,
\[
(\theta^\star)^\top A\theta^\star
= \sum_{i=1}^n (x_i^\top \theta^\star)^2
\le \sum_{i=1}^n \|x_i\|_2^2 \|\theta^\star\|_2^2
\le nR^2 r^2.
\]
Combining the bounds yields
\[
\|A^{1/2}\theta\|_2^2
\le \bigl(\sqrt{n}\,rR + \sqrt{rR}\bigr)^2
\le R_B^2,
\]
where $R_B := \sqrt{n}\,rR + \sqrt{rR}$. Therefore, $\theta \in B$, and hence
$\mathcal E_{rR} \subseteq B$.
Since $\mu_B$ is the normalized Lebesgue measure on $B$ and $\mathcal E_{rR}^<\subseteq \mathcal{E}_{rR}\subseteq B$,
\[
 \mu_B(\mathcal E_{rR}^<)
= \frac{\mathrm{Vol}(\mathcal E_{rR}^<)}{\mathrm{Vol}(B)} \ge \frac{1}{2}\cdot\frac{\mathrm{Vol}(\mathcal E_{rR})}{\mathrm{Vol}(B)} =\frac{1}{2}\cdot\left( \frac{\sqrt{rR}}{R_B} \right)^d.
\]
By \cref{lem:contain-ellipsoid-logistic},
\[\mu_B(\cH_{rR}) \ge\mu_B(\mathcal E_{rR}^<)\]
Substituting the definition of $R_B$ gives
\[
\mu_B(\cH_{rR})
\;\ge\;
\frac{1}{2}\cdot\left(
\frac{\sqrt{rR}}{\sqrt{n}\,rR + \sqrt{rR}}
\right)^d
=
\frac{1}{2}\cdot\left(
\frac{1}{\sqrt{n rR} + 2}
\right)^d.
\]
Since $\sqrt{n rR} + 2 \le 4\vee n rR$, we obtain
\[
\mu_B(\cH_{rR})
\ge
\frac{1}{2}\cdot(4 \vee n rR)^{-d} \ge ( 8\vee 2nrR)^{-d} 
\]
Taking logarithms yields
\[
\log \frac{1}{\mu_B(\cH_{rR})}
\;\le\;
d \log (8\vee 2n r R),
\]
which completes the proof.
\end{proof}

We now verify the level-set growth condition \cref{ass:grid-key}.

\begin{lemma}
\label{lem:local-growth-logistic}
Assume that $\|x_i\|_2 \le R$, $\cH = \{\theta \in \bR^d : \|\theta\|_2 \le r\}$. Let the level sets
$\cH_t$ and $\cH_{t,i}$ be defined as in \eqref{eq:logistic-level-sets}.
Let $\mu_B$ be uniform on $B$ as above, $\Delta =  1+rR + \sqrt{\frac{rR}{\lambda_{\min}(A)}}\cdot R$, and $\cT := \{\Delta,2\Delta,3\Delta,\dots,16\,\Delta\cdot d\log(8\vartriangleleft2nrR)\}$.
Then $(\cH,\ell_\logit)$ satisfies \cref{ass:grid-key} with parameters
\[
(\mu,\cT,\Delta,C_g,\rho) = (\mu_B,\cT, rR + \sqrt{\frac{rR}{\lambda_{\min}(A)}}\cdot R,2,3/4).
\]
\end{lemma}

\begin{proof}[\pfref{lem:local-growth-logistic}]
Let $\vartheta \in \cH_A$. By definition, there exists
$\theta \in \mathcal H$ such that
\[
\|\vartheta - \theta\|_A^2 \le rR.
\]
Since $A \succeq \lambda_{\min}(A) I$, we have
\[
\|\vartheta - \theta\|_A^2
\ge
\lambda_{\min}(A)\|\vartheta - \theta\|_2^2.
\]
Hence
\[\|\vartheta - \theta\|_2
\le
\sqrt{\frac{rR}{\lambda_{\min}(A)}}.
\]
By the triangle inequality,
\[
\|\vartheta\|_2
\le
\|\theta\|_2 + \|\vartheta - \theta\|_2.
\]

Since $\theta \in \cH$, we have $\|\theta\|_2 \le r$.
Therefore,
\[
\|\vartheta\|_2
\le
r + \sqrt{\frac{rR}{\lambda_{\min}(A)}}.
\]

Since $\|x_i\|_2\le R$ and $\|\theta\|_2\le r + \sqrt{\frac{rR}{\lambda_{\min}(A)}}$ on $\cH_A$, the logistic loss is uniformly bounded by
$1+ rR + R\sqrt{\frac{rR}{\lambda_{\min}(A)}}$ on $\cH_A$.       
By \cref{lem:level-set-sandwich}, for every $i \in [n]$ and every tolerance level $t \ge 0$,
\[
\cH_{t-\Delta} \subseteq \cH_{t,i} \subseteq \cH_{t+\Delta}.
\]
with $\Delta =1+ rR + \sqrt{\frac{rR}{\lambda_{\min}(A)}}\cdot R$. So it suffices to
control the growth ratio $\mu_B(\cH_{\tol+\Delta})/\mu_B(\cH_{\tol-\Delta})$.
Fix a tolerance level $t \in \cT$. The local level-set growth condition \cref{ass:key} with $\Delta=1+rR + \sqrt{\frac{rR}{\lambda_{\min}(A)}}\cdot R$ requires
\[
\frac{\mu_B(\cH_{t+\Delta})}{\mu_B(\cH_{t-\Delta})} \le C_g.
\]
We now show that this inequality holds with $C_g=2$ for all but at most $4 d \log(8\vee n r R)$ values of $t$.
Define the set of \emph{bad} levels
\[
B := \{ t \in \cT : \mu_B(\cH_{t+\Delta}) > 2\,\mu_B(\cH_{t-\Delta}) \}, 
\qquad s := |B|.
\]
Order $B$ as $t_1 < t_2 < \dots < t_s$. Consider the subsequence of odd indices $t_1, t_3, t_5, \dots$.
Iterating along this subsequence yields
\[
\mu_B(\cH_{t_{2\ell-1}+\Delta}) > 2^{\ell-1} \,\mu_B(\cH_{t_1+\Delta}),
\qquad \ell = 1, \dots, \lceil s/2 \rceil.
\]
Since $\Delta= 1 + rR + \sqrt{\frac{rR}{\lambda_{\min}(A)}}\cdot R \ge rR$, we have
\[
\cH_{t_1+\Delta} \supseteq \cH_{\Delta} \supseteq \cH_{rR}
\]
Hence,
\[
\mu_B(\cH_{t_s+\Delta}) > 2^{\lceil s/2 \rceil - 1} \, \mu_B(\mathcal{H}_{rR}).
\]
Since $\mu_B$ is a probability measure, $\mu_B(\cH_{t_s+\Delta}) \le 1$, so
\[
2^{\lceil s/2 \rceil - 1} < \frac{1}{\mu_B(\mathcal{H}_{rR})}.
\]
By Lemma~\ref{lem:logistic-volume-ratio},
\[
\log \frac{1}{\mu_B(\mathcal{H}_{rR})} \le d \log(8\vee 2n r R),
\]
which implies
\[
\lceil s/2 \rceil - 1 \le d \log(8\vee 2n r R),
\qquad \text{hence } s \le 2 d \log(8\vee 2n r R) + 2 < 4 d \log(8\vee 2n r R).
\]
Since $|\cT| = 16\, d \log(2n r R)$, at least
\[
|\cT| - s \ge \frac{3}{4} |\cT|
\]
levels satisfy
\[
\frac{\mu_B(\cH_{t+\Delta})}{\mu_B(\cH_{t-\Delta})} \le 2.
\]
Therefore, $(\cH, \ell_\logit)$ satisfies the level-set growth condition with parameters
\[
(\mu, \cT, \Delta, C_g, \rho) = (\mu_B, \cT, 1+rR + \sqrt{\frac{rR}{\lambda_{\min}(A)}}\cdot R, 2, 3/4).
\]
\end{proof}

Combining \cref{lem:local-growth-logistic} with \cref{thm:main} yields the following
leave-one-out oracle inequality.

\begin{corollary}
\label{cor:loo-logistic}
Assume that $\|x_i\|_2 \le R$, $\cH = \{\theta \in \bR^d : \|\theta\|_2 \le r\}$. 
Let the level sets
$\cH_t$ and $\cH_{t,i}$ be defined as in \eqref{eq:logistic-level-sets}.
Let $\mu_B$ be uniform on $B$, $\Delta = 1+ rR + \sqrt{\frac{rR}{\lambda_{\min}(A)}}\cdot R$, let $\cT=\{\Delta,2\Delta,\dots,16\,\Delta\cdot d\log(8\vee 2nrR)\}$.
Let $\{\hat p_i\}_{i=1}^n$ be the output of \cref{alg:levelset-aggregation} when the inner
aggregation is \eqref{eq:logistic-agg}.
Then
\[
\mathrm{LOO}_S(\{\hat p_i\}_{i\in [n]})
\;\le\;
\frac{8}{n} \min_{\theta\in \cH} L_S(\theta) + \frac{136}{n}\left(1+ rR + \sqrt{\frac{rR}{\lambda_{\min}(A)}}\cdot R\right)\, d \log(8\vee 2n r R).
\]
\end{corollary}

Corollary~\ref{cor:loo-logistic} yields a leave-one-out oracle inequality for logistic
regression over a bounded parameter class.
Previously, comparable guarantees were obtained by \citet[Corollary~2]{mourtada2022improper},
who showed that the Ridge SMP estimator satisfies
$\mathrm{LOO}_S \le \tfrac{1}{n}\min_{\theta\in\cH}L_S(\theta)+\tfrac{ed+r^2R^2}{n}$.
In regimes where $rR\gg d$, $\lambda_{\min}(A)$ is not too small, and the empirical risk $\min_{\theta\in\cH} L_S(\theta)$ is small, our bound exhibits a sharper dependence on the problem parameters.

\begin{proof}[\pfref{cor:loo-logistic}]
By \cref{lem:local-growth-logistic}, $(\cH, \ell_\logit)$ satisfies the level-set growth condition
\cref{ass:grid-key} with parameters
\[
(\mu, \cT, \Delta, C_g, \rho) = (\mu_B, \cT,  1+ rR + \sqrt{\frac{rR}{\lambda_{\min}(A)}}\cdot R, 2, 3/4).
\]
Applying Theorem~\ref{thm:main} gives
\[
\mathrm{LOO}_S(\{\hat p_i\}_{i\in [n]})
\le
\frac{2 C_g}{(2 \rho - 1)n} \bigl( L_S(\theta^\star) + \maxtol + \Delta \bigr),
\]
Substituting $C_g = 2$, $\rho = 3/4$, $\maxtol = 16\, \Delta\, d \log(8\vee 2n r R)$, and $\Delta = 1+ rR + \sqrt{\frac{rR}{\lambda_{\min}(A)}}\cdot R$ yields
\begin{align*}
\mathrm{LOO}_S(\{\hat p_i\}_{i\in [n]})
&\le \frac{8}{n} \bigl( \min_{\theta\in \cH} L_S(\theta) + 17\left(1+ rR + \sqrt{\frac{rR}{\lambda_{\min}(A)}}\cdot R\right)\, d \log(8\vee 2n r R) \bigr) 
\end{align*}
as claimed.
\end{proof}

\section{Conclusion}

We introduced Median of Level-Set Aggregation (MLSA), a two-layer procedure for transductive leave-one-out prediction that aggregates over near-ERM level sets and uses a median over tolerances to robustify the prediction. Under a simple local level-set growth condition, we proved a general multiplicative oracle inequality for the LOO error on arbitrary fixed datasets. We verified this condition in several canonical problems—VC classification, bounded convex regression, log-loss density estimation (with smoothing), and logistic regression via a geometric/volumetric argument—yielding standard complexity terms such as $O(d\log n)$ or $O(\log|\cH|)$.

\clearpage

{
\bibliography{refs}
}

\appendix

\section{A Transductive Variant of Vovk--Azoury--Warmuth}

\label{app:loo-for-vovk}

The classical Vovk--Azoury--Warmuth (VAW) predictor \citep{vovk2001competitive,azoury2001relative} is an \emph{online} square-loss
regression algorithm derived from the Aggregating Algorithm.
At round $t$, its prediction depends only on previously observed data and is
computed using the prefix Gram matrix
$A_{t}=aI+\sum_{s\leq t} x_s x_s^\top$
(and the corresponding linear term $b_{t-1}=\sum_{s<t} x_s y_s$),
yielding a prediction of the form
$\hat y_t = x_t^\top A_{t}^{-1} b_{t-1}$, up to minor variants.

The result in this appendix concerns a different, transductive construction.
Rather than operating online, we work with a predictor defined using the
full-sample Gram matrix
$A=\sum_{j=1}^n x_j x_j^\top$.
Accordingly, the bound proved here does not analyze the VAW predictor itself.
Instead, it establishes a leave-one-out guarantee for a related linear predictor
that replaces the prefix Gram matrices appearing in the online setting with the
terminal Gram matrix, which is natural when the covariates are treated as fixed.

\begin{theorem}\label{thm:loo-bound-linear}
Let \( (x_i, y_i)_{i=1}^n \) be data with \( x_i \in \mathbb{R}^d \), \( y_i \in \mathbb{R} \).  
Define the full-sample OLS estimator $\hat\beta = A^{-1}\sum_{j=1}^n x_jy_j$.
For each $i$, define the \emph{shrinkage leave-one-out} estimator
\[
\hat\beta_{-i} := A^{-1}\sum_{j\ne i} x_jy_j,
\]
which removes the $i$-th contribution from the linear term while keeping the same Gram
matrix $A$.

Assume \( y_i^2 \leq m^2 \) for all \( i \). Then the total leave-one-out error satisfies:
\[
\sum_{i=1}^n \left( y_i - x_i^T \hat{\beta}_{-i} \right)^2 \leq 2 \sum_{i=1}^n \left( y_i - x_i^T \hat{\beta} \right)^2 + 2 m^2 d.\]
\end{theorem}

\begin{proof}[\pfref{thm:loo-bound-linear}]
By definition, we have:

\[
y_i - x_i^T \hat{\beta}_{-i} = y_i - x_i^T \hat{\beta} + x_i^T A^{-1} x_i \, y_i.
\]

Squaring both sides and using the inequality \( (a+b)^2 \leq 2a^2 + 2b^2 \):

\[
\left( y_i - x_i^T \hat{\beta}_{-i} \right)^2 \leq 2 \left( y_i - x_i^T \hat{\beta} \right)^2 + 2 \left( x_i^T A^{-1} x_i \right)^2 y_i^2.
\]

Since \( y_i^2 \leq m^2 \) and \( x_i^T A^{-1} x_i \leq 1 \), we obtain:

\[
\left( y_i - x_i^T \hat{\beta}_{-i} \right)^2 \leq 2 \left( y_i - x_i^T \hat{\beta} \right)^2 + 2 m^2 \, x_i^T A^{-1} x_i.
\]

Summing over \( i = 1, \dots, n \):

\[
\sum_{i=1}^n \left( y_i - x_i^T \hat{\beta}_{-i} \right)^2 \leq 2 \sum_{i=1}^n \left( y_i - x_i^T \hat{\beta} \right)^2 + 2 m^2 \sum_{i=1}^n x_i^T A^{-1} x_i.
\]

Now, observe that:

\[
\sum_{i=1}^n x_i^T A^{-1} x_i = \sum_{i=1}^n \operatorname{Tr}\left( x_i^T A^{-1} x_i \right)
= \operatorname{Tr}\left( A^{-1} \sum_{i=1}^n x_i x_i^T \right)
= \operatorname{Tr}\left( A^{-1} A \right)
= \operatorname{Tr}(I_d)
= d.
\]

Thus,

\[
\sum_{i=1}^n \left( y_i - x_i^T \hat{\beta}_{-i} \right)^2 \leq 2 \sum_{i=1}^n \left( y_i - x_i^T \hat{\beta} \right)^2 + 2 m^2 d.
\]
\end{proof}

\end{document}